\documentclass{article}

\usepackage{microtype}
\usepackage{graphicx}
\usepackage{subfigure}
\usepackage{booktabs} 
\usepackage[accepted]{icml2020}

\usepackage{amsthm}
\usepackage{amssymb}
\usepackage{amsmath}

\usepackage{hyperref}

\newtheorem{theorem}{Theorem}
\newcommand{\etal}{\textit{et al}.}
\newcommand{\ie}{\textit{i}.\textit{e}.}
\newcommand{\eg}{\textit{e}.\textit{g}.}
\renewcommand{\algorithmicrequire}{\textbf{Input:}}
\renewcommand{\algorithmicensure}{\textbf{Output:}}
\DeclareMathOperator*{\argmax}{argmax}

\icmltitlerunning{Adversarially Robust Frame Sampling with Bounded Irregularities}

\begin{document}

\twocolumn[
\icmltitle{Adversarially Robust Frame Sampling with Bounded Irregularities}



\icmlsetsymbol{equal}{*}

\begin{icmlauthorlist}
\icmlauthor{Hanhan Li}{google}
\icmlauthor{Pin Wang}{berkeley}
\end{icmlauthorlist}

\icmlaffiliation{google}{Google AI, 1600 Amphitheatre Parkway, Mountain View, CA 94043}
\icmlaffiliation{berkeley}{California PATH, UC Berkeley, Richmond, CA, 94804}

\icmlcorrespondingauthor{Hanhan Li}{uniqueness@google.com}

\icmlkeywords{Computer Vision}

\vskip 0.3in
]



\printAffiliationsAndNotice{}  

\begin{abstract}
In recent years, video analysis tools for automatically extracting meaningful information from videos are widely studied and deployed. Because most of them use deep neural networks which are computationally expensive, feeding only a subset of video frames into such algorithms is desired. Sampling the frames with fixed rate is always attractive for its simplicity, representativeness, and interpretability. For example, a popular cloud video API generated video and shot labels by processing only the first frame of every second in a video. However, one can easily attack such strategies by placing chosen frames at the sampled locations. In this paper, we present an elegant solution to this sampling problem that is provably robust against adversarial attacks and introduces bounded irregularities as well.
\end{abstract}

\section{Introduction}
The amount of videos we produce each year is growing at an incredible pace. The videos may originate from various kinds of devices such as professional video recorders, personal cameras, surveillance cameras, smartphones, \etal, or they may be created completely from software. Internet companies host an enormous amount of videos for people to watch online. Video analysis algorithms, ranging from handcrafted ones such as shot boundary detection \cite{boreczky1996comparison,pal2015video} to machine learned ones such as convolutional neural networks \cite{krizhevsky2012imagenet,he2016deep}, also advanced significantly. Machine learning algorithms are especially attractive these days as they are capable of performing difficult tasks such as extracting semantic meaning from raw media. Companies are increasingly reliant on these video understanding tools to better filter, index, and rank videos for search and recommendation at scale, and these topics are widely studied in both academia and industry. 

However, most of the video analysis algorithms are computationally expensive. For example, it requires $7.6\times 10^9$ Mul-Add FLOPs to apply a ResNet101 model on a single frame with a $224\times224$ resolution \cite{he2016deep}. One may achieve better performance with a larger ResNet model or with a higher frame resolution, where even more FLOPs are required. Therefore, analyzing all frames can be unaffordable or cost-inefficient, and sampling a subset of frames beforehand is usually desired.

Generally, we prefer to sample as uniformly as possible for the following three reasons. First, a uniform sequence better represents the whole video. Second, it is better interpretable and explainable. Third, irregular sequences may add extra complexity to or degrade the performance of downstream algorithms. There is a well known cloud service that takes a video file from a user and returns the video labels (objects within the video), shot changes (scene changes within the video), shot labels (description of video events over time), and more. The original service processed only the first frame of every second of a video to save computing budgets. However, if images were inserted at the rate of one frame per second into a video, the API would only output video and shot labels related to the inserted images only and ignore the rest which is the vast majority. This vulnerability was discovered and demonstrated by Hosseini \etal \cite{hosseini2017attacking}.

One needs to introduce randomness into the sampling algorithm as a countermeasure to these image insertion attacks. However, it will necessarily compromise uniformity, and we would like to keep the disturbance as small as possible. This paper provides such a solution, named `jittering with reflection', that is provably robust and has bounded irregularities. As frame timestamps of a video can be treated as either continuous or discrete, we will address both variants in this paper. To our best knowledge, there is no prior work that jointly optimizes both uniformity and adversarial robustness of frame sampling.

Please note that there is an orthogonal problem on pixel-level robustness associated with deep neural networks \cite{szegedy2013intriguing} that this paper does not address. The problem there is that certain imperceptible perturbation to the images can trick the network into making completely wrong predictions.

The rest of the paper is organized as follows. In Section \ref{sec:continuous}, we present the theoretical formulation and solution to the continuous version of the sampling robustness problem. Its discrete counterpart is addressed in Section \ref{sec:discrete}. Section \ref{sec:example} demonstrates an example sampling with its associated video classification performance, and Section \ref{sec:conclusions} concludes this paper.

\section{The continuous version}
\label{sec:continuous}
In this section, we will first define mathematically the desired uniformity properties and randomness properties as well as the rationale behind the definitions. We will then prove that these properties are sufficient to ensure security, \ie, robustness against insertion attacks. Finally, we will propose a sampling algorithm and prove that it satisfies all the desired properties. Time is treated as a continuous quantity as it is in the physical world.

\subsection{Uniformity and randomness properties}
For a given interval $t\in \mathbb{R}^+$ and a perturbation threshold $t_p\in (0, t)$, we would like to probabilistically sample an infinite sequence $A = \{a_0, a_1, a_2, \dotsc \}$ in $\mathbb{R}^+$. We require the following uniformity properties:
\begin{itemize}
    \item[U1.] $|a_{i+1} - a_i -t | \leq t_p$ for any $i \geq 0$.
    \item[U2.] There exists some offset $o \in \mathbb{R}$ such that $|a_i - (it + o)| \leq t/2$ for any $i \geq 0$.
\end{itemize}
We also require the following randomness properties:
\begin{itemize}
    \item[R1.] There exists some threshold $\delta > 0$ such that $\delta < \mathrm{p}(r \in A) < \infty$ for any $r \in \mathbb{R}^+$.
    \item[R2.] The event $r + q \in A$ becomes independent of all events $s \in A$ for $s < r$ as $q \to \infty$. Formally, for any $\epsilon > 0$, there exists a $q$ such that $|\mathrm{p}(r + q \in A) - \mathrm{p}(r + q \in A | W \cap A = W')| < \epsilon$ for any $r > 0$, $W \subset [0, r)$, and $W' \subset W$.
\end{itemize}
where $\mathrm{p}(r \in A)$ is the probability density of the event $r \in A$. Intuitively, the probability of $[r, r+\mathrm{d}r)$ intersecting with $A$ is $\mathrm{p}(r \in A)\mathrm{d}r$.

In our application, the sequence $A$ represents the desired timestamps with which we would like to take samples with a target frequency $1/t$. The threshold $t_p$ is usually chosen to be much smaller than $t$.

The uniformity properties ensure that the sampled frames are well-represented and interpretable for a normal video, as they are evenly spaced up to a bounded error. Property U1 is about incremental uniformity, and it ensures that the time intervals between two neighboring frames are reasonably close to $t$. This is particularly important for algorithms that infer motion or depth from the difference between the contents of these two frames \cite{dosovitskiy2015flownet,gordon2019depth}. Property U2 is about cumulative uniformity, and it ensures that there is no long time deviation from a fixed frame rate. This is particularly important if we need to align the sampled frames with a stream of features from another modality (\eg audio) that may have a constant frequency.

The randomness properties ensure that the sampled frames are robust against adversarially positioned frames. Property R1 ensures that there is no blind spot that we never take sample from. Property R2 ensures that there is no long time correlation that could be exploited. For example, if a uniform sequence is shifted by an overall random interval, R1 could be satisfied, but R2 could not. In this case, an attacker at most needs to try a couple of times to find an offset to his/her favor and trick the sampling strategy.

The following subsection proves that any set of frames that persistently appear throughout a video will have a probability exponentially close to $1$ of getting caught in a sampled sequence with the above properties.

\subsection{Security statement and proof}
\begin{theorem}
Let $S$ be a subset of $\mathbb{R}^+$ with a measure of $\infty$.  Let $A = \{a_0, a_1, a_2, \dotsc \}$ be a sampled sequence in $\mathbb{R}^+$ that satisfies the uniformity and randomness properties. Then the probability of $S \cap [0, u)$ and $A$ being disjoint goes to $0$ exponentially with the measure of $S \cap [0, u)$ as $u$ goes to $\infty$.
\end{theorem}

\begin{proof}
Let $\delta$ be the threshold in Property R1. By Property R2, there exists an integer $n$ such that
\begin{align*}
    \big|&\mathrm{p}\big(r + (n-1)t \in A\big) \\
    &- \mathrm{p}\big(r + (n-1)t \in A | W \cap A = W'\big)\big| < \delta / 2 
\end{align*}
for any $r > 0$, $W \subset [0, r)$, and $W' \subset W$. It follows that
\begin{equation}\label{probabilitythreshold}
\mathrm{p}\big(r + (n-1)t \in A | W \cap A = W'\big) > \delta / 2
\end{equation}
for any $r$, $W$, and $W'$ under the same condition.

We pick an offset $o'$ such that $o' - t / 2\leq 0$ and that $o-o'$ is an integer multiple of $t$. We partition $[o' - t/2, \infty)$ into the following $n$ sets, $I_0$, $I_1$, ... $I_{n-1}$, where
\begin{equation*}
    I_i = \bigcup_{k=0}^{\infty}I_{i,k}
\end{equation*}
with the intervals defined as 
\begin{equation*}
  I_{i,k}=\big[(i+kn-1/2)t + o', (i+kn+1/2)t + o'\big)
\end{equation*}

Let us denote $I_{i,k}^{u} = [0, u)\cap I_{i,k}$ and $I_i^{u} = [0, u)\cap I_i$. Then $I_0^u$, $I_1^u$, ... $I_{n-1}^u$ is a partition of $[0, u)$.

It's clear that
\begin{equation} \label{probabilitythresholdcondition}
    |x-x'|>(n-1)t
\end{equation}
for any $x \in I_{i,k}^{u}$ and  $x' \in I_{i,k'}^{u}$ if $k \ne k'$.

Let $l(u)=\lceil (u-o'+t/2)/(nt)\rceil$, it's clear that $I_{i,k}^{u}=\emptyset$ if $k>l(u)$.

Let $\mu$ be the measure function on $R$. We define
\begin{equation*}
   j_u=\argmax_{i \in \{0, 1, ...n-1\}}\mu\big(S\cap I_i^u\big)
\end{equation*}
It's clear that 
\begin{equation*}
    \mu\big(S\cap I_{j_u}^u\big) \geq \mu\big(S\cap[0, u)\big)/n
\end{equation*}

By Property U2, with a probability of $1$, there is exactly one $a_m$ in the interval $I_{j_u,k}^u$, and hence the events $x \in A$ for all $x$ in $I_{j_u,k}^u$ are disjoint.

As a result, by Property R1,
\begin{align*}
    & \mathrm{P}(S \cap I_{j_u,0}^u \cap A = \emptyset) \\
    = & 1 - \int_{S \cap I_{j_u,0}^u} \mathrm{p}(x \in A) \mathrm{d}x \\
    \leq & 1 - \delta \mu(S \cap I_{j_u,0}^u)
\end{align*}

In addition, by Equation \ref{probabilitythreshold} and Equation \ref{probabilitythresholdcondition}, for any $k \geq 1$,
\begin{align*}
    &\mathrm{P}\big(S \cap I_{j_u,k}^u \cap A = \emptyset \\
       & \quad \big| S \cap I_{j_u,k'}^u \cap A = \emptyset \text{ for } k' \in \{0, 1,...k-1\}\big) \\
    =& 1 - \int_{ S \cap I_{j_u,k}^u} \\
         & \mathrm{p}\big(x \in A \big| S \cap I_{j_u,k'}^u \cap A = \emptyset \text{ for } k' \in \{0, 1,...k-1\}\big) \mathrm{d}x \\
    \leq & 1 - \frac{1}{2}\delta \mu(S \cap I_{j_u,k}^u)
\end{align*}

Therefore,
\begin{align*}
         &\mathrm{P}(S \cap[0,u) \cap A = \emptyset) \\
    \leq & \mathrm{P}(S \cap I_{j_u,k}^u \cap A = \emptyset \text{ for } k \ \in \{0,1,...l(u)\}) \\
    =&\mathrm{P}(S \cap I_{j_u,0}^u \cap A = \emptyset) \cdot\\
     &\prod_{k=1}^{l(u)}\mathrm{P}\big(S \cap I_{j_u,k}^u \cap A = \emptyset \\
       & \quad\quad\quad \big| S \cap I_{j_u,k'}^u \cap A = \emptyset \text{ for } k' \in \{0, 1,...k-1\}\big) \\
    \leq& \big(1 - \delta \mu(S \cap I_{j_u,0}^u)\big)  \prod_{k=1}^{l(u)} \big( 1 - \frac{1}{2}\delta \mu(S \cap I_{j_u,k}^u) \big) \\
    \leq &\exp\big(-\frac{1}{2}\delta\sum_{k=0}^{l(u)}\mu(S \cap I_{j_u,k}^u)\big)  \\
    = & \exp\big(-\frac{1}{2}\delta \mu(S \cap I_{j_u}^u)\big) \\
    \leq & \exp\big(-\frac{\delta \mu\big(S\cap[0, u)\big)}{2n}\big)
\end{align*}

Therefore,  the probability $\mathrm{P}(S \cap[0,u) \cap A = \emptyset)$ goes to $0$ as $u$ goes to $\infty$, and the rate of convergence is exponential with $\mu\big(S\cap[0, u)\big)$.
\end{proof}

\subsection{Jittering with reflection}
Our proposed sampling strategy ‘Jittering with Reflection’ satisfies all desired properties. Moreover, it has an elegant behavior that $\mathrm{p}(r \in A)$ is constantly
$1 / t$ , which maximally ensures robustness. The key of the algorithm is mirror reflections against the boundaries of the intervals $[i  t, (i + 1)t]$. Figure \ref{fig:frame_sampling} shows examples of the sampled frames.

\begin{figure}[h]
\begin{center}
\includegraphics[width=0.9\linewidth]{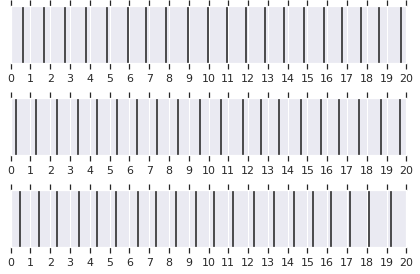}
\end{center}
   \caption{This figure shows sampled frames from Algorithm \ref{alg:continuous} with 3 different seeds. We used $t=1$ and $V = \mathrm{Unif}([-0.1, 0.1])$. Locations of the vertical lines represent times of the sampled frames.}
\label{fig:frame_sampling}
\end{figure}

\begin{algorithm}[h]
\caption{Jittering with reflection (continuous version)}
\begin{algorithmic}
\label{alg:continuous}
\REQUIRE A jittering distribution $V$ on $[-t_p, t_p]$ which is symmetric about $0$. It has a piecewise smooth p.d.f.. For example $V = \mathrm{Unif}([-t_p, t_p])$.
\ENSURE A sampled sequence $\{a_0, a_1, a_2, \dotsc\}$.
\STATE Sample $a_0 \sim \mathrm{Unif}([0, t])$.
\FOR{$i = 1, 2, ... $}
\STATE Sample $v_i \sim V$, set $b_i = a_{i - 1} + t + v_i$.
\IF{$b_i > (i + 1)t$}
\STATE Set $a_i = 2(i + 1)t - b_i$.
\ELSIF{$b_i < it$}
\STATE Set $a_i = 2it - b_i$.
\ELSE
\STATE Set $a_i = b_i$.
\ENDIF
\ENDFOR
\end{algorithmic}
\end{algorithm}

\begin{theorem}
The sampled sequence from Algorithm \ref{alg:continuous} satisfies the uniformity and randomness properties.
\end{theorem}

\begin{proof}
It is clear that uniformity properties U1 and U2 are satisfied, with offset $o = t / 2$.

Denote $b_i = a_i - it$, then $b_i \in [0, t]$. Moreover, $\{b_i | i = 0, 1, \dotsc\}$ is a Markov chain with transitions given by
\begin{subequations}
\begin{align}
    b_i &= b_{i - 1} + v_i, &\text{ if } b_{i - 1} + v_i \in [0, t], \label{case1}\\
    b_i &= -(b_{i - 1} + v_i), &\text{ if } b_{i - 1} + v_i < 0, \label{case2}\\
    b_i &= 2t - (b_{i - 1} + v_i), &\text{ if } b_{i - 1} + v_i > T. \label{case3}
\end{align}
\end{subequations}

Let $q_i(x)$ be the probability density function of $b_{i - 1} + v_i$, and $r_i(x)$ be the probability density function of $b_i$. We prove by induction that $b_i \sim \mathrm{Unif}([0, t])$ for any $i$. $b_0$ follows the distribution by design. Suppose $b_{i-1}$ follows it, we break the range of $b_i$ into three segments.

If $t_p \leq x \leq t - t_p$, only the case in Equation \ref{case1} can happen:
\begin{align*}
    r_i(x) &= q_i(x) \\
           &= \int_{-t_p}^{t_p} r_{i - 1}(x - v) \mathrm{p}(v_i = v) \mathrm{d}v \\
           &= \int_{-t_p}^{t_p} (1/t) \mathrm{p}(v_i = v)\mathrm{d}v \\
           &= 1/t
\end{align*}
If $0 \leq x < t_p$, only the cases in Equation \ref{case1} and Equation \ref{case2} can happen:
\begin{align*}
    r_i(x) =& q_i(x) + q_i(-x) \\
           =& \int_{-t_p}^{x} r_{i - 1}(x - v) \mathrm{p}(v_i = v) \mathrm{d}v  \\
           & + \int_{-t_p}^{-x} r_{i - 1}(-x - v) \mathrm{p}(v_i = v) \mathrm{d}v \\
           =& \int_{-t_p}^{x} r_{i - 1}(x - v) \mathrm{p}(v_i = v) \mathrm{d}v \\
           & + \int_{x}^{t_p} r_{i - 1}(v - x) \mathrm{p}(v_i = -v) \mathrm{d}v \\
           =& \int_{-t_p}^{t_p} (1/t) \mathrm{p}(v_i = v)\mathrm{d}v \\
           =& 1/t
\end{align*}
We used the fact that $\mathrm{p}(v_i = v) = \mathrm{p}(v_i = -v)$ because $V$ is symmetric about $0$. The same result for $t - t_p < x \leq t$ can be proved in a similar way. It follows that $b_i \sim \mathrm{Unif}([0, t])$, which concludes the induction.

As the sequence $A$ partitions $\mathbb{R}^+$ into intervals and each $a_i$ covers a unique interval $[it, (i+1)t]$ with uniform distribution, Property R1 is satisfied.

To prove Property R2, we will show that the Markov chain ${b_i}$ with any initial distribution will converge to the same stationary distribution with an exponential rate. First, we define $s_i \colon \mathbb{R} \to \mathbb{R}$ that has a period of $2t$, with
\begin{align*}
    s_i(x) = r_i(x), & \text{ if } x \in [0, t), \\
    s_i(x) = r_i(-x), & \text{ if } x \in [-t, 0)
\end{align*}
And we look at the Fourier series of $s_i$ with period $2t$:
\begin{equation*}
    s_i(x) = \sum_{k = -\infty}^{\infty} c_{i,k}\exp (\mathrm{j}\pi kx/t)
\end{equation*}
and the Fourier transform of $\mathrm{p}(v_i = x)$:
\begin{equation*}
    \mathrm{p}(v_i = x) = \frac{1}{2t}\int_{-\infty}^{\infty} d(\omega) \exp (\mathrm{j}\pi \omega x/t) \mathrm{d}\omega 
\end{equation*}
It is easy to see that $s_i$ has a nice behavior under `jittering with reflection'. Namely, $s_i(x)$ is the convolution of $s_{i-1}(x)$ and $\mathrm{p}(v_i = x)$, or
\begin{equation*}
    c_{i,k} = c_{i-1,k}d(k)
\end{equation*}
By induction,
\begin{equation*}
    c_{i,k} = c_{0,k}d(k)^i
\end{equation*}
Since
\begin{equation*}
    d(\omega) = \int_{-\infty}^{\infty} \mathrm{p}(v_i = x) \exp (-\mathrm{j}\pi \omega x/t) \mathrm{d}x
\end{equation*}
we see that $|d(k)| \leq 1$, with the equality holds only if $k = 0$. Also, by Riemann–Lebesgue lemma, we have $|d(k)| \to 0$ as $k \to \infty$. Therefore, when $i \to \infty$, only the zero frequency coefficient $c_{i,0}$ will survive, and $s_i(x)$ becomes a constant function regardless of $s_0(x)$. Also, the convergence has bounded exponential rate. This asymptotic independence implies that Property R2 is satisfied.

\end{proof}

\section{The discrete version}
\label{sec:discrete}

In this section, we are going to formulate the discrete version of this problem, which is relevant to real products. For example, MediaPipe \cite{lugaresi2019mediapipe} is an open sourced framework for   building multi-modal (\eg video, audio) applied ML pipelines, and its timestamps are at $1$ microsecond granularity.

\subsection{Uniformity and randomness properties}
For a given interval $t \in \mathbb{N}$ and a perturbation threshold $t_p \in \{1, 2, ..., t - 1\}$, we would like to probabilistically sample an infinite sequence $A = \{a_0, a_1, a_2, \dotsc \}$ in $\mathbb{N}$ that satisfies both the uniformity properties and randomness properties as defined below.

\begin{itemize}
    \item[U1.] $|a_{i+1} - a_i -t | \leq t_p$ for any $i \geq 0$.
    \item[U2.] There exists some offset $o \in R$ such that $|a_i - (it + o)| \leq t/2$ for any $i \geq 0$.
    \item[R1.] There exists some threshold $\delta > 0$ such that $\mathrm{P}(r \in A) > \delta$ for any $r \in \mathbb{N}$.
    \item[R2.] The event $r + q \in A$ becomes independent of all events $s \in A$ for $s < r$ as $q \to \infty$. Formally, for any $\epsilon > 0$, there exists an integer $q$ such that $|\mathrm{P}(r + q \in A) - \mathrm{P}(r + q \in A | W \cap A = W')| < \epsilon$ for any integer $r > 0$, $W \subset \{0, 1, \dotsc{}, r-1\}$, and $W' \subset W$.
\end{itemize}

\subsection{Security Statement}
\begin{theorem}
Let $S$ be an infinite subset of $\mathbb{N}$.  Let $A = \{a_0, a_1, a_2, \dotsc \}$ be a sampled sequence in $\mathbb{N}$ that satisfies the uniformity and randomness properties. Then the probability of $S \cap \{0, 1, 2,...u-1\}$ and $A$ being disjoint goes to $0$ exponentially with the cardinality of $S \cap \{0, 1, 2,...u-1\}$ as $u$ goes to $\infty$.
\end{theorem}

The proof will be similar to the continuous version and is omitted.

\subsection{Jittering with reflection}
\begin{algorithm}[h]
\caption{Jittering with reflection (discrete version)}
\begin{algorithmic}
\label{alg:discrete}
\REQUIRE A jittering distribution $V$ on $\{-t_p, -t_p + 1, \dotsc{}, t_p - 1, t_p\}$ which is symmetric about $0$. The gcd of $2t$ and the indices of the nonzero entries of $V$ needs to be $1$. For example $V = \mathrm{Unif}(\{-t_p, -t_p + 1, \dotsc{}, t_p - 1, t_p\})$. 
\ENSURE A sampled sequence $\{a_0, a_1, a_2, \dotsc\}$.
\STATE Sample $a_0 \sim \mathrm{Unif}(\{0, 1, \dotsc, t - 1\})$.
\FOR{$i = 1, 2, ... $}
\STATE Sample $v_i \sim V$, set $b_i = a_{i - 1} + t + v_i$.
\IF{$b_i \geq{} (i + 1)t$}
\STATE Set $a_i = 2(i + 1)t - b_i - 1$.
\ELSIF{$b_i < it$}
\STATE Set $a_i = 2it - b_i - 1$.
\ELSE
\STATE Set $a_i = b_i$.
\ENDIF
\ENDFOR
\end{algorithmic}
\end{algorithm}

The discrete `jittering with reflection' sampling is stated in Algorithm \ref{alg:discrete}. 
\begin{theorem}
  The sampled sequence from Algorithm \ref{alg:discrete} satisfies the uniformity and randomness properties.
\end{theorem}

The proof will be analogous to that of the continuous version. We can show that $a_i \sim \mathrm{Unif}(\{it, it +1, \dotsc, (i + 1)t - 1\})$ for any $i$. The additional condition $\gcd (\{2t\} \cup \{k | \mathrm{P}(v = k) > 0\}) = 1$ ensures that the Markov chain $\{a_i - it|i=0,1,\dotsc\}$ will converge to $\mathrm{Unif}(\{0, 1, 2, \dotsc, t - 1\})$ for any initial distribution. The reason is that, in the discrete Fourier transform of $\mathrm{P}(v = k)$, the coefficients of all but the constant term are less than $1$.

\section{An example: video classification after frame sampling}
\label{sec:example}
In this section, we demonstrate a simple example of the discrete frame sampling with interval $T=2$ and presents its impact on video classification.

Let $x_i\equiv a_i-2i$, then $\{x_i|i=0,1,...\}$ is a Markov chain with only two possible states, $0$ and $1$. At each step, the jittering flips the state with a probability of $\alpha$. The transition matrix $P$, where $P_{ij}$ represents the probability of moving from state $i$ to state $j$, is given by
\begin{equation*}
    P=
    \begin{vmatrix}
        1-\alpha&\alpha\\
        \alpha&1-\alpha\\
    \end{vmatrix}
\end{equation*}

It can be derived that the correlation between $x_i$ and $x_{i+m}$ is
\begin{equation*}
  \mathrm{corr}(x_i, x_{i+m}) = (1-2\alpha)^m
\end{equation*}
It shows that the correlation drops to zero at an exponential speed, which is consistent with the randomness properties. We define the correlation length $l_c$ as the number of steps over which the correlation drops to $1/\mathrm{e}$, \ie,
\begin{equation*}
  l_c = -\frac{1}{\log(1-2\alpha)}
\end{equation*}
A larger $\alpha$ gives a shorter correlation length and therefore more robustness. The red curve in Figure \ref{fig:experiment} visualizes this relation.

\begin{figure}[h]
\begin{center}
\includegraphics[width=0.9\linewidth]{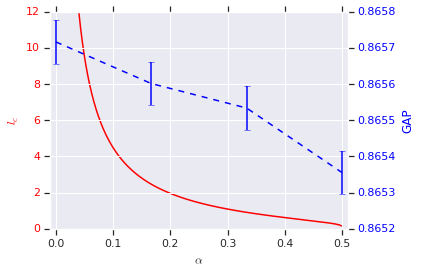}
\end{center}
   \caption{This figure presents some numerical results of the discrete `jittering with reflection' sampling with $T=2$. $\alpha$ in the horizontal axis is a measure of the jittering magnitude, and the variance of the distance between two neighboring frames is proportional to it. The red line shows the correlation length of sampled frames decreases with $\alpha$. The blue line shows that the performance (as measured in GAP) of a Youtube-8M classification model slightly decreases with $\alpha$.}
\label{fig:experiment}
\end{figure}

It can also be derived that the variance of $\delta a_i = a_{i+1}-a_i$ is
\begin{equation*}
  \mathrm{Var}(\delta a_i) = \alpha
\end{equation*}
This means a larger $\alpha$ indicates further departure from uniformity. Sampling frames at irregular intervals may have many undesired consequences. Nevertheless, we use the following experiment to demonstrate that its impact on video level classification tasks is small.

YouTube-8M \cite{abu2016youtube} is a large-scale labeled video dataset that consists of features from millions of YouTube videos with high-quality machine-generated annotations. We use the 2018 version which has about 6 million videos with a diverse vocabulary of
3862 audio-visual entities. 1024-dimensional visual features and 128-dimensional audio features at 1 frame per second are extracted from bottleneck layers of pre-trained deep neural networks and are provided as input features for this dataset.

We train a deep-bag-of-frames (DBoF) model as described in Li \etal \cite{li2019ensemblenet}. In the DBoF, a few layers (shared across frames) are applied to each frame, and then the frame level features are aggregated into a video feature, and finally a few additional layers are applied to obtain predictions. For evaluation, we only take frames at multiples of 5 seconds just to magnify the effect of sampling. On top of these frames, we use the discrete `jittering with reflection' with $T=2$ (which corresponds to 10 seconds in the videos) to sample frames before applying the DBoF model. The blue curve in Figure \ref{fig:experiment} shows that the global average precision (GAP) of the DBoF model slightly degrades when $\alpha$ increases, which is expected. The variation is small because topical annotations are insensitive to locations of the frames.

One can adjust this $\alpha$ to achieve the desired trade-off between robustness (as measured by $l_c$) and uniformity (as measured by $\mathrm{Var}(\delta a_i)$).

\section{Conclusions}
\label{sec:conclusions}
In this paper, we formulated the uniformity and randomness properties that a general frame sampling strategy desires. We proved that if these properties are satisfied, a strategy is robust in the sense that any recurring sequence of frames has exponentially small chance of concealing itself. We designed an algorithm `jittering with reflection' that satisfies all the desired properties, and the magnitude of the jittering can be tuned to achieve the desired trade-off between the amount of irregularities and the degree of robustness. We expect this algorithm to be widely useful in video analysis.

\section*{Acknowledgements}
We thank Qingchun Ren for assistance in formalizing many proofs in this paper.


\end{document}